%% file: archive.tex
\documentclass[runningheads]{llncs}
\usepackage[T1]{fontenc}
\usepackage{graphicx}
\usepackage{booktabs}
\usepackage[misc]{ifsym}
\usepackage{amsmath}
\usepackage{amssymb}
\usepackage{multirow}
\usepackage{makecell}
\usepackage{adjustbox}   
\usepackage{siunitx}
\usepackage{threeparttable}
\usepackage{graphicx}
\usepackage{subfig}
\usepackage{bm}

\usepackage{color}

\usepackage{mwe}
\input{mysymbol.sty}

\begin{document}

\title{Train to Stitch: Topology-Preserving Regularization for Interoperable Representations}
\author{Akshay Malhotra\inst{1}\thanks{Equal Contribution.} \and
Nicolas Zilberstein\inst{2*} \and
Shahab Hamidi-Rad\inst{1}}

\institute{InterDigital Communications Inc., Los Altos, CA, USA 
\and 
Rice University, Houston, TX, USA 
}
\titlerunning{Train to Stitch}
\maketitle

\begin{abstract}
As large, powerful pre-trained models become more prevalent, there is growing interest in developing methods that allow different neural networks to be seamlessly combined with little or no direct communication between them. 
Take autoencoder-based data compression as an example: the ability to independently choose architectures and train encoder and decoder models, while still ensuring they work together, could fundamentally transform how such models are built, deployed, and updated. 
In this work, we introduce a novel approach that leverages topological regularization to align the latent spaces of two independently trained autoencoder models, entirely eliminating the need for coordination during training.
We provide a conditional theoretical analysis showing how the structural agreement encouraged by our regularizer can translate into approximate geometric compatibility between independently learned representations.
Reconstruction experiments validate this approach empirically, substantially narrowing the interoperability gap across multiple datasets.
\end{abstract}

\section{Introduction}
\label{sec:intro}

Compressing data into compact, meaningful representations is a fundamental problem with broad applications spanning image compression~\cite{image1,image2}, audio compression~\cite{audio1}, wireless communication~\cite{csinet,csi_comp2}, and beyond. 
A common neural network approach trains an encoder-decoder pair end-to-end, where the encoder maps the data into a low-dimensional latent space, and the decoder reconstructs the original input from this latent representation.
In many practical settings, however, encoders and decoders cannot be trained jointly. 
For example, a client-side encoder may face strict hardware or latency constraints while the server-side decoder is updated frequently, or separate organizations may build their own encoders and decoders without sharing model weights or training data. 
These scenarios call for \emph{interoperable} designs: encoder and decoder components that are trained independently yet can be connected at deployment time with little or no coordination.

This challenge has been explored under a particular instantiation of this problem, termed \emph{zero-shot stitching}~\cite{gygli2021towards}: given two independently trained autoencoders, the goal is to find a transformation between latent spaces from only a small calibration set of shared inputs.
An early work~\cite{moschellarelative} proposed mapping each autoencoder's latent space into a shared relative representation.
A key limitation, however, is that this requires training a dedicated decoder in the relative space, which demands substantial amounts of data. Other recent approaches~\cite{lahner2024direct,maiorca2024latent} have investigated whether a simple affine mapping can align independently learned latent spaces. 
While promising, a notable accuracy gap persists compared to jointly trained autoencoders, largely because each encoder's nonlinear mapping is shaped by its architecture, optimization hyperparameters, and random initialization, producing latent geometries that can differ considerably across models. This effect is exacerbated in the small-data regime, and closing the gap requires large calibration sets, contradicting the premise of minimal interaction and limited shared data.

To close this gap, we take a complementary approach: rather than relying on the existence of a simple post-hoc, linear map, we introduce a training-time inductive bias that encourages independently trained models to produce \emph{geometrically compatible} latent spaces. 
Our key insight is that many alignment assumptions can be understood through the lens of (approximate) \emph{isometry}: if two point clouds are nearly distance-preserving copies of the same underlying structure, then an orthogonal alignment is well-justified (Section~\ref{sec:theory}). 
This motivates regularizing training so that each latent space preserves the same \emph{geometric skeleton} of the data.

To instantiate this idea, we build on topological autoencoders~\cite{moor2020topological} and employ a differentiable topological loss that aligns $0$-dimensional persistence features between a target space and a learnable space.\footnote{For $0$-dimensional persistent homology, the persistence pairing corresponds to the edges of the minimum spanning tree~\cite{hensel2021survey}.} 
Concretely, this topological loss encourages agreement of minimum spanning tree (MST) edge lengths between two spaces by leveraging the persistence pairings in each space. Intuitively, MST edges form a sparse topological backbone that captures global connectivity; matching this backbone across models yields latent spaces that are less sensitive to architectural idiosyncrasies and hence easier to stitch.

Building on this foundation, we propose a strategy that regularizes each autoencoder so that its latent space preserves the topological structure of the input space, requiring no information exchange between model owners (Fig.~1). 
After training, we fit a simple linear transformation using only a small calibration set.
We further provide a theoretical analysis explaining why this regularization yields compatible latent spaces.
We validate our approach through several numerical experiments spanning image compression and channel compression and reconstruction.

\paragraph{Contributions.} The main contributions of this paper are summarized as follows:
\begin{enumerate}
\setlength\itemsep{0.5em}
\item \textbf{Topology-preserving training for interoperability.} We introduce a topological regularization framework for independently trained autoencoders that promotes latent-space compatibility using only 0-dimensional persistence features (MST edges). This loss requires no model-to-model interaction during training and enables post-training model stitching via a simple linear map learned from a small calibration set.

\item \textbf{Theoretical motivation for linear stitching.} We show that reference-MST edge agreement controls disagreement between the induced tree-path metrics. Under the additional assumption of bounded tree stretch, we derive a bound on Euclidean pairwise-distance disagreement, providing a geometric motivation for fitting a low-complexity linear connector (Section~\ref{sec:theory}).

\item \textbf{Extensive multi-domain evaluation.} We validate our approach on reconstruction and retrieval tasks across heterogeneous architectures. Our method demonstrates substantial reductions in stitching error compared to baseline direct alignment and relative representations, achieving near-optimal performance across multiple datasets.

\end{enumerate}

\begin{figure*}[t]
    \centering
    \includegraphics[width=1\textwidth]{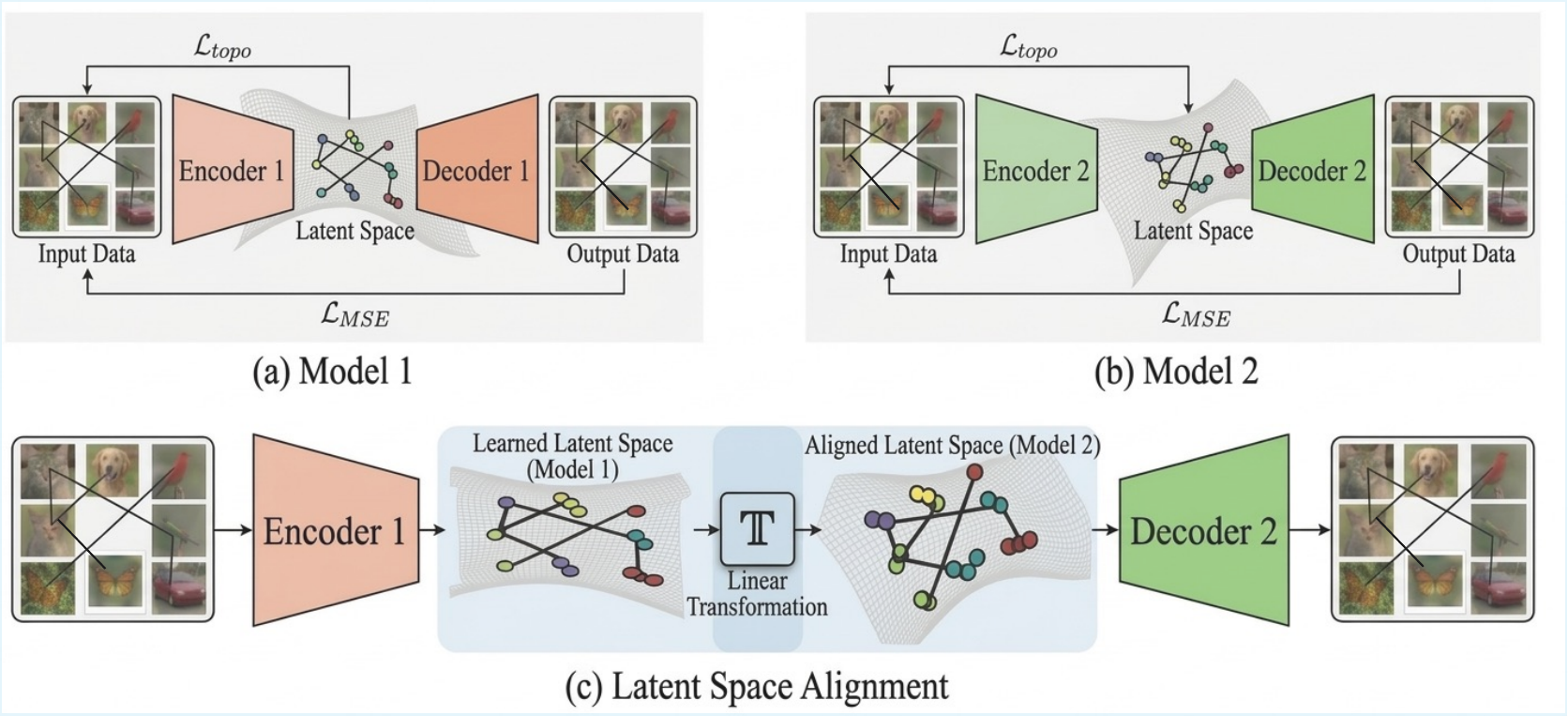}
	\caption{\textbf{Proposed scheme}: (a-b) \emph{Topological autoencoders.}
We regularize the training of each autoencoder $\mathrm{AE}_1, \mathrm{AE}_2$ with the input data, preserving its topological structure in the latent space independently of the architecture of the encoder.
We build graphs in both input and latent spaces, and we leverage the topological regularization to preserve the underlying connectivity.
 Notice that each autoencoder maps the input data to a different latent space, but both latent spaces preserve the topology up to a rotation and scaling.
 c) \emph{Stitching between two regularized autoencoders}. The transformation is simplified, leading to an affine transformation.}
    \label{fig:scheme1}
\end{figure*}

\section{Related Work}

This work is primarily concerned with the interoperability of independently trained autoencoders, building on two lines of research: zero-shot stitching and geometric regularization of autoencoders.

\subsection{Zero-shot stitching}

Relative representation~\cite{moschellarelative} is one of the most popular papers in zero-shot stitching.
The work builds on the observation that the representations learned by different neural networks trained on the same data are related via conformal maps; in other words, the angles between latent embeddings are preserved.
Based on this observation, they propose to map the latent space of each autoencoder to a pre-defined relative representation, which is invariant by
construction to the transformations induced by stochastic factors in the training process.
More precisely, they select at random a set of anchor elements $\ccalA \in \ccalD_{train}$;
every sample in the anchor set is represented concerning the embedded anchors $\boldsymbol{e}_{\boldsymbol{a}^{(j)}}=\ccalE\left(\boldsymbol{a}^{(j)}\right)$.
Then, the relative representation is defined as
\begin{equation}
\begin{aligned}
\boldsymbol{r}_{\boldsymbol{x}^{(i)}}=\big(&
\operatorname{sim}(\boldsymbol{e}_{\boldsymbol{x}^{(i)}},
                   \boldsymbol{e}_{\boldsymbol{a}^{(1)}}),
\operatorname{sim}(\boldsymbol{e}_{\boldsymbol{x}^{(i)}},
                   \boldsymbol{e}_{\boldsymbol{a}^{(2)}}),\\
&
\ldots, \operatorname{sim}(\boldsymbol{e}_{\boldsymbol{x}^{(i)}},
                   \boldsymbol{e}_{\boldsymbol{a}^{(|\ccalA|)}})\big)
\end{aligned}
\end{equation}
where $\operatorname{sim}: \mathbb{R}^d \times \mathbb{R}^d \rightarrow \mathbb{R}$ is a similarity function, yielding a scalar score $r$ between two absolute representations $r=\operatorname{sim}\left(\boldsymbol{e}_{\boldsymbol{x}^{(i)}}, \boldsymbol{e}_{\boldsymbol{x}^{(j)}}\right)$.
In particular, they consider the cosine similarity.
Notice that the number of anchor points gives the dimensionality of this new similarity space; therefore, it is expected to improve the performance when increasing the number of anchor points.
This proposed method requires a specialized decoder trained on the relative representation.

While in the original paper, the authors considered only the cosine similarity, other similarity functions can be incorporated~\cite{cannistracibricks,cannistraci2023infusing}.
A recent work~\cite{fumero2024latent} proposes a method to do zero-shot stitching that leverages functional maps, i.e., by learning a linear transformation in the spectral domain.

A closely related line of work \cite{lahner2024direct,maiorca2024latent} studies post-hoc alignment of independently learned latent spaces using simple parametric maps, enabling ``zero-shot'' translation between models without retraining the encoders/decoders.

\subsection{Autoencoders via geometric regularization}

This work focused on topological autoencoders; however, several previous works incorporate a geometric regularization~\cite{lee2023geometric}.
Many of them build on the notion of \emph{curvature} of a generator~\cite{arvanitidis2018latent}.
In~\cite{lee2021neighborhood}, the authors propose first to build a neighborhood graph on the input data and use it to regularize the geometry
and connectivity of the learned manifold.
Remarkably, they propose a local approximation of the decoder instead of the encoder to extract local geometric information on the decoded manifold.
In~\cite{yonghyeon2021regularized}, a regularization based on conformal mapping, i.e., a mapping that preserves angles, is used to enforce the preservation of angles and relative distances between data and latent space.
In~\cite{duque2022geometry}, inspired by manifold learning, a regularization based on spectral embeddings is incorporated, encouraging the learned latent representation to align with the data geometry.
Recently, a graph-preserving autoencoder was proposed in~\cite{limgraph}, which aims to preserve a pre-build Laplacian between the input data and the latent space.

\section{Alignment with topological constraints}
\label{sec:method}


\subsection{Problem formulation}
\label{sec:problem_form}
We consider two autoencoders, $\mathrm{AE}_1=(E_1,D_1)$ and $\mathrm{AE}_2=(E_2,D_2)$, where $E_i$ and $D_i$ denote the encoder and decoder of autoencoder $i$, respectively. Both autoencoders are trained \emph{independently} on the same data distribution; that is,
they may differ in architecture, initialization, and optimization trajectory.
Given an input $x \in \reals^l$, each encoder produces a latent code $z_i = E_i(x) \in \mathbb{R}^d$.
Our goal is to enable interoperability between these models: connecting the encoder of one with the decoder of the other.
Formally, we seek to perform \emph{zero-shot stitching}; that is, cross-model decoding of the form $D_2(T(E_1(x)))$, or symmetrically $D_1(T(E_2(x)))$, where $T(\cdot)\colon\mathbb{R}^d\rightarrow\mathbb{R}^d$ is an affine transformation learned from only a small calibration set with minimal information exchange between the two models.

\paragraph{Direct alignment between latent spaces.}
Recent works~\cite{lahner2024direct,maiorca2024latent} proposed to directly align the latent spaces by fitting a transformation $T$ using a small dataset $D_T$ of shared inputs.
Denoting the latent embeddings of the calibration set as
\[
\bbZ_1 := \{E_1(x)\mid x\in D_T\},\qquad \bbZ_2 := \{E_2(x)\mid x\in D_T\},
\]
the transformation is estimated by minimizing the mean squared discrepancy
between the two point clouds:
\begin{equation}
\label{eq:direct-alignment}
T(\cdot) = \arg\min_{T(\cdot)\in\mathcal{T}} \|\bbZ_1 - T(\bbZ_2)\|^2.
\end{equation}

These methods typically restrict $\mathcal{T}$ to affine or orthogonal maps, which offers simplicity but leaves a significant performance gap compared to end-to-end training. 
This suggests that the latent spaces produced by unconstrained, independent training are not well-aligned by such maps: each encoder induces a nonlinear geometry that is strongly influenced by architecture, optimization, and random initialization, making $Z_1$ and $Z_2$ poorly compatible under a simple affine fit.

Conversely, we view the existence of a low-complexity alignment map as a
\emph{geometric consequence} of the two latent point clouds being approximately
isometric. 
Recall the classical result that if two centered point clouds have
identical pairwise distance matrices, then they are related by an orthogonal
transformation. This is formalized in Lemma~\ref{lem:1}.

\begin{lemma}
\label{lem:1}
Let $\mathbf{Z},\mathbf{Y}\in\mathbb{R}^{n\times d}$ be matrices whose rows are the points
$\{z_i\}_{i=1}^n$ and $\{y_i\}_{i=1}^n$ and assume they are centered:
$\mathbf{Z}^\top \mathbf{1}=\mathbf{Y}^\top \mathbf{1}=0$.
Let $\Delta(\mathbf{Z})\in\mathbb{R}^{n\times n}$ denote the squared distance matrix,
$\Delta(\mathbf{Z})_{ij}=\|z_i-z_j\|_2^2$ (and similarly for $\Delta(\mathbf{Y})$).
If $\Delta(\mathbf{Z})=\Delta(\mathbf{Y})$, then there exists $\mathbf{Q}\in\mathbb{R}^{d\times d}$
with $\mathbf{Q}^\top \mathbf{Q}=I$ such that $\mathbf{Y}=\mathbf{Z}\mathbf{Q}$.
\end{lemma}

\begin{proof} See Appendix A.1. \nocite{borg2005modern}
\end{proof}

Lemma~\ref{lem:1} motivates why a rigid alignment --- either orthogonal or affine --- is a natural choice when the two representations associated to each latent space $\bm{Z_1}$ and $\bm{Z_2}$ are approximately isometric.
However, independent training does not guarantee this property.
Our proposed framework aims to address this gap: we introduce a topological regularizer that explicitly encourages each encoder to produce latent spaces with compatible geometry, making a simple post-hoc alignment viable.


\subsection{Topological regularization to promote alignability}
\label{sec:topo-reg}

We now introduce our proposed framework for training autoencoders with compatible latent spaces.
The key idea is to \emph{promote geometric compatibility} between independently trained latent spaces by adding a geometric constraint during training. Specifically, we encourage the latent space to preserve the topological structure of a \emph{reference} point cloud via a differentiable topological loss $L_{\mathrm{topo}}(\cdot,\cdot)$ based on persistent homology \cite{moor2020topological}.
In particular, we focus on $0$-dimensional persistence features, for which the persistence pairing boils down to the edges of the MST. Intuitively, the MST's edges form a sparse ``topological skeleton'' capturing global connectivity of the point cloud; matching these edges across spaces is a lightweight way to encourage comparable geometry.

\paragraph{Topological loss.}
Given a target space $X:=\{x_i\}_{i=1}^n$ and a learnable space $Z:=\{z_i\}_{i=1}^n$, let $P$ be the set
of all unordered pairs and $A_X$ (resp.\ $A_Z$) the vector of all pairwise distances in $X$ (resp.\ $Z$).
Let $P_X$ and $P_Z$ denote the persistence pairing indices in the target and learnable spaces (MST edges
when $d=0$). Denoting
\[
\begin{aligned}
D_{XX}&:=A_X[P_X], & D_{ZZ}&:=A_Z[P_Z],\\
D_{ZX}&:=A_Z[P_X], & D_{XZ}&:=A_X[P_Z],
\end{aligned}
\]
the topological loss is
\begin{equation}
\label{eq:topo-loss}
L_{\mathrm{topo}}(X,Z)
= \tfrac12\|D_{XX}-D_{ZX}\|_F^2 + \tfrac12 \|D_{XZ}-D_{ZZ}\|_F^2.
\end{equation}

\subsubsection{MST-based topological regularization promotes linear alignability.}
\label{sec:theory}

Our loss \eqref{eq:topo-loss} only constrains a \emph{sparse} subset of distances (MST edges).
We next show that while \eqref{eq:topo-loss} does not, in general, guarantee isometry, it encourages geometric agreement along a shared MST backbone; under additional measurable conditions, this can be sufficient to imply approximate near-isometry.
The first term of \eqref{eq:topo-loss} enforces agreement between target MST edge lengths and those same
edges measured in the learnable space.



Let $T_X=(V,P_X)$ denote the MST graph on vertex set $V=\{1,\dots,n\}$ with edge set $P_X$.
For a point cloud $Z$, define the \emph{tree-path metric} induced by $T_X$:
\[
d_{T_X}^{Z}(i,j) := \sum_{(u,v)\in\mathrm{path}_{T_X}(i,j)} \|z_u-z_v\|_2,
\]
where $\mathrm{path}_{T_X}(i,j)$ is the unique path between $i$ and $j$ in the tree (as it is an MST).

Let $\mathrm{diam}(T_X)$ be the tree diameter in number of edges:
$\mathrm{diam}(T_X):=\max_{i,j}|\mathrm{path}_{T_X}(i,j)|$.

\begin{lemma}[Edge mismatch controls all tree-path distances]
\label{lem:tree_metric}
For any two point clouds $Z$ and $Y$ on the same index set and any tree $T_X=(V,P_X)$,
\[
\begin{aligned}
\max_{i,j} \big| d_{T_X}^{Z}(i,j) - d_{T_X}^{Y}(i,j) \big|
&\le \sqrt{\mathrm{diam}(T_X)}\\
&\quad {}\times\|A_Z[P_X]-A_Y[P_X]\|_2.
\end{aligned}
\]
where $A_Z[P_X]$ is the vector of edge lengths $\{\|z_u-z_v\|_2 : (u,v)\in P_X\}$ (and similarly for $A_Y[P_X]$).
\end{lemma}


\begin{proof} See Appendix A.2.
\end{proof}

Lemma~\ref{lem:tree_metric} yields an \emph{unconditional} form of compatibility.
Specifically, if two latent spaces approximately match the same reference MST edge lengths on $P_X$
(equivalently, the edge residuals $\varepsilon_1,\varepsilon_2$ are small), then they induce similar
\emph{tree-path metrics} on $T_X$: for all pairs $(i,j)$,
$d^{Z_1}_{T_X}(i,j)\approx d^{Z_2}_{T_X}(i,j)$, with an explicit bound provided by
Lemma~\ref{lem:tree_metric}. This tree-metric agreement is the quantity most directly controlled by our
topological regularizer.

Further, the approximate agreement of \emph{Euclidean} pairwise distances
$\|z^{(1)}_i-z^{(1)}_j\|_2 \approx \|z^{(2)}_i-z^{(2)}_j\|_2$ (near-isometry), requires an additional
condition relating the tree-path metric to Euclidean distance in each latent space.
We capture this via a bounded stretch assumption (Definition~\ref{def:stretch}).


\begin{definition}[Tree stretch with respect to $T_X$]
\label{def:stretch}
For a point cloud $Z$, define the stretch
\[
\tau(Z;T_X) := \max_{i\neq j} \frac{d_{T_X}^{Z}(i,j)}{\|z_i-z_j\|_2}\in[1,\infty].
\]
\end{definition}

Small $\tau(Z;T_X)$ means the MST-path metric is a good proxy for Euclidean distances in $Z$.
This quantity is dataset-dependent but can be estimated by sampling pairs.

Bounded tree stretch is \emph{not} guaranteed by the MST in the worst case; we therefore treat it as a
measurable diagnostic. Under bounded stretch, tree-metric agreement can be lifted to approximate Euclidean distance agreement, quantified by a bound on $\|\Delta(\mathbf{Z}_1)-\Delta(\mathbf{Z}_2)\|_F$
(Theorem~\ref{thm:mst_to_linear}).


\begin{theorem}[MST/topological regularization promotes linear stitchability]
\label{thm:mst_to_linear}
Let $\mathbf{Z}_1,\mathbf{Z}_2\in\mathbb{R}^{n\times d}$ be latent point clouds (rows are the points
$\{z^{(1)}_i\}_{i=1}^n$ and $\{z^{(2)}_i\}_{i=1}^n$) and assume they are centered:
$\mathbf{Z}_1^\top 1=\mathbf{Z}_2^\top 1=0$.
Fix a reference tree $T_X=(V,P_X)$ with vertex set $V=\{1,\dots,n\}$ and edge set $P_X$
(the $0$-dimensional persistence pairing in the data space). For $k\in\{1,2\}$ define the MST edge-mismatch residual
\[
\varepsilon_k := \big\|A_{Z_k}[P_X]-A_X[P_X]\big\|_2,
\]
Assume a bounded tree-stretch condition holds for both clouds (Definition~\ref{def:stretch}):
$\tau(Z_1;T_X)\le\tau$ and $\tau(Z_2;T_X)\le\tau$.

Then the squared-distance matrices satisfy,

\[
\begin{aligned}
\|\Delta(\mathbf{Z}_1)-\Delta(\mathbf{Z}_2)\|_F
&\le 2R\sqrt{n(n-1)}\Bigl(2(\tau-1)R\\
&\qquad {}+\sqrt{\mathrm{diam}(T_X)}
(\varepsilon_1+\varepsilon_2)\Bigr).
\end{aligned}
\]
Here, $R := \max_{i\neq j}\max\{\|z^{(1)}_i-z^{(1)}_j\|_2,\ \|z^{(2)}_i-z^{(2)}_j\|_2\}$.
\end{theorem}

\begin{proof} See Appendix Section A.3.
\end{proof}

Combined with the result of Lemma~\ref{lem:1}, the approximate Euclidean distance agreement established by Theorem~\ref{thm:mst_to_linear} provides geometric motivation for seeking a low-distortion orthogonal alignment, although it does not by itself provide a quantitative alignment guarantee. Overall, our analysis provides (i) an unconditional guarantee on the induced tree-metric
alignment and (ii) a conditional route to Euclidean near-isometry in regimes
where stretch is moderate.


\paragraph{Discussion and limitations.}
We want to emphasize that Theorem~\ref{thm:mst_to_linear} provides only a \emph{conditional guarantee}, and
MST constraints alone do not uniquely determine a latent embedding, but they provably enforce agreement of induced tree metrics (Lemma~\ref{lem:tree_metric}). When the learned representations additionally have a small tree stretch (Definition~\ref{def:stretch}), this implies approximate pairwise-distance agreement, which provides geometric motivation for seeking an orthogonal alignment.

\subsection{Training methodology and stitching}
\label{sec:schemes}

We incorporate \eqref{eq:topo-loss} into the training framework of the independently trained autoencoders, allowing them to produce more compatible latent spaces. Each autoencoder is regularized so that its latent space preserves the topological structure of the \emph{input} space. Concretely, for $\mathrm{AE}=(E,D)$, we train with
\begin{equation}
\label{eq:d2l}
\mathit{L} = \|x - D(E(x))\|^2 + \lambda\, L_{\mathrm{topo}}(x,E(x)).
\end{equation}
Because both $\mathrm{AE}_1$ and $\mathrm{AE}_2$ are trained to preserve the same reference topology (the data geometry induced by the input metric), their latent spaces are encouraged to share a common topological skeleton, despite being learned by different architectures. Note that this loss does not require any information to be exchanged between the two models $\mathrm{AE}_1$ and $\mathrm{AE}_2$ for training.

\paragraph{Stitching after training.}
After training, we fit an affine transformation $T$ using a
small calibration set $D_T$ by solving \eqref{eq:direct-alignment} with $T(z)=Tz$.
At inference time, stitching is performed by composing the encoder, the learned transformation,
and the decoder, e.g.,
\[
\hat x = D_1(T(E_2(x))),
\]
and symmetrically for the opposite direction.


\section{Experiments}

We evaluate the performance of the proposed topology-preserving training and investigate its efficacy in achieving interoperability across independently trained
models/representations when the models are connected post hoc by a lightweight affine connector. We report the performance across multiple different use cases and data domains: (i) image compression (section \ref{sec:results_images}), (ii) wireless communication channel compression (section \ref{sec:results_csi}), (iii) sketch-based image retrieval performance (section \ref{sec:results_fscoco}). Additionally, in section \ref{sec:results_geometry} we present an empirical evaluation of the geometric diagnostics that probe the conditions used in our analysis in section \ref{sec:theory}.

\subsection{Stitched image compression and reconstruction}
\label{sec:results_images}

We consider an autoencoder-based compression framework in which two architecturally different autoencoder models are first trained independently and then, post-training, cross-model performance is evaluated. The evaluation is performed across \textbf{MNIST} \cite{deng2012mnist}, \textbf{CIFAR-10} \cite{krizhevsky2009learning}, \textbf{Tiny-ImageNet} \cite{le2015tiny} and \textbf{CelebA} \cite{liu2015faceattributes} datasets. As baseline metrics, we report the end-to-end performance of the two models ($\mathrm{L_{11}}$ and $\mathrm{L_{22}}$, i.e. encoder and decoder coming from the same jointly trained autoencoder without any regularization) and the direct cross model performance ($\mathrm{L_{12}}$ and $\mathrm{L_{21}}$, i.e. encoder and decoder coming from the different autoencoders and joined directly, without any form of alignment). Our method is indicated by $\mathrm{L^{Topo}}$. We compare the performance of our approach with the relative representation-based stitching ($\mathrm{L^{RR}}$) \cite{moschellarelative}, and the direct affine alignment-based stitching ($\mathrm{L^{DAA}}$) proposed in \cite{maiorca2024latent,lahner2024direct}. To ensure that the latent space alignment across independently trained models is not an artifact of the model architecture, we choose completely different model architectures for the two autoencoders. For the MNIST and CIFAR-10 datasets, the first model is an MLP-based model, and the other has a convolutional neural network (CNN) based architecture. In the Tiny-ImageNet and CelebA datasets, the first model is a CNN based model and the other model is a vision transformer (ViT) based model. All experimental results are averaged over 5 independent training runs with different initializations. Further details on the model architecture and training can be found in Appendix B.

\begin{table*}[t]
\centering
\caption{The table shows the interoperability results on image datasets. We compare the reconstruction NMSE (dB) for different methods across different encoder-decoder configurations. $\mathrm{L_{ij}}$ denotes the use of encoder $E_i$ with decoder $D_j$.}
\label{tab:results_images}
\begin{adjustbox}{max width=\linewidth}
\begin{threeparttable}
\begin{tabular}{l *{10}{S[table-format=-2.2]}} 
\toprule
\textbf{Dataset} &
{$\mathbf{L_{11}}$} & {$\mathbf{L_{22}}$} & {$\mathbf{L_{12}}$} & {$\mathbf{L_{21}}$} &
{\makecell{$\mathbf{L^{DAA}_{1\rightarrow T\rightarrow 2}}$ }} &
{\makecell{$\mathbf{L^{DAA}_{2\rightarrow T\rightarrow 1}}$ }} &
{\makecell{$\mathbf{L^{RR}_{12}}$ }} &
{\makecell{$\mathbf{L^{RR}_{21}}$ }} &
{\makecell{$\mathbf{L^{Topo}_{1\rightarrow T\rightarrow 2}}$}} &
{\makecell{$\mathbf{L^{Topo}_{2\rightarrow T\rightarrow 1}}$}} \\
\midrule
\textbf{MNIST}      & -15.38 & -16.64 & -0.11 & 1.19  & -11.09 & -11.16 & -2.23 & -2.98 & {\textbf{-12.81}} & {\textbf{-12.42}} \\
\textbf{CIFAR-10}   & -14.09 & -15.38 & -2.99 & -0.36 & -13.80 & -13.82 & -5.87 & -6.68 & {\textbf{-14.34}} & {\textbf{-14.35}} \\
\textbf{T-ImageNet} & -11.73 & -10.42 & 0.05  & -0.01 & -10.22 & -11.27 & -7.59 & -7.01 & {\textbf{-10.48}} & {\textbf{-11.28}} \\
\textbf{CelebA}     & -15.36 & -15.33 & -2.57 & -4.5  & -14.38 & -14.46 & -7.92 & -5.55 & {\textbf{-14.93}} & {\textbf{-14.92}} \\
\bottomrule
\end{tabular}
\end{threeparttable}
\end{adjustbox}
\end{table*}

As seen from Table~\ref{tab:results_images} the results for \emph{direct cross-model usage}, $\mathrm{L_{12}}$ and $\mathrm{L_{21}}$, are highly unreliable across all
image datasets. This reaffirms the observations from prior work that independently trained latent spaces are generally not directly compatible.

Fitting an \emph{affine} connector using a small calibration set substantially improves interoperability across all datasets. For example, on MNIST, the affine-stitched errors drop to $\mathrm{L^{DAA}_{1\rightarrow T\rightarrow 2}}=-11.09$ and $\mathrm{L^{DAA}_{2\rightarrow T\rightarrow 1}}=-11.16$, representing more than an order-of-magnitude improvement over direct cross-decoding. Similar behavior holds on the other datasets too.

Topology-preserving training yields an additional consistent gain over affine stitching.
On MNIST, topology reduces the affine-stitched error by $1.7$ dB for $L_{1\rightarrow T\rightarrow 2}$ and $1.25$ dB for $L_{2\rightarrow T\rightarrow 1}$.
Across all datasets, the $\mathrm{L^{Topo}}$ outperforms the existing alignment approaches, although in the Tiny-ImageNet data, the improvements are less pronounced, suggesting that in this setting, the affine connector already captures most of the alignable structure in that direction.



\subsection{CSI compression: interoperability beyond vision}
\label{sec:results_csi}

We next evaluate our approach on a compression task from wireless communications, where we compress the \emph{channel state information} (CSI). CSI is a structured, high-dimensional representation of the wireless channel; in real systems, including 5G systems, it is communicated under strict bandwidth constraints, making compression essential. The use of autoencoders is being extensively studied for the upcoming releases of the wireless communication standards \cite{lin2023overview}. From an ML perspective, CSI compression is analogous to learned image compression: an encoder maps an input tensor to a compact code, which is quantized and a decoder reconstructs the original signal, with quality measured by reconstruction error.

We consider 3GPP standards-compliant CSI datasets \textbf{CDL-A,C} and \textbf{UMa} \cite{3gpp38901}, where each sample is a complex-valued CSI tensor. We represent CSI using two real-valued channels (real and imaginary parts) and train the autoencoder with a quantization block (2-bit uniform scalar quantization \cite{malhotra2024exploiting}) after the encoding stage. Similar to the image experiments, 2 different model architectures (CsiNet based \cite{csinet}, and transformer based \cite{malhotra2024exploiting}) are trained independently. We evaluate interoperability by stitching encoders and decoders across model architectures.

\begin{table*}[t]
\centering
\caption{Results on channel state information (CSI) compression. Same notation as Table~\ref{tab:results_images}.}
\label{tab:results_csi}
\begin{adjustbox}{max width=\linewidth}
\begin{threeparttable}
\begin{tabular}{l S[table-format=1.4] S[table-format=1.4] S[table-format=1.4] S[table-format=1.4] S[table-format=1.4] S[table-format=1.4] S[table-format=1.4] S[table-format=1.4] S[table-format=1.4] S[table-format=1.4]}
\toprule
\textbf{Dataset} &
{$\mathbf{L_{11}}$} & $\mathbf{L_{22}}$ & $\mathbf{L_{12}}$ & $\mathbf{L_{21}}$ &
{\makecell{$\mathbf{L^{DAA}_{1\rightarrow T\rightarrow 2}}$ }} &
{\makecell{$\mathbf{L^{DAA}_{2\rightarrow T\rightarrow 1}}$ }} &
{\makecell{$\mathbf{L^{RR}_{12}}$}} &
{\makecell{$\mathbf{L^{RR}_{21}}$}} &
{\makecell{$\mathbf{L^{Topo}_{1\rightarrow T\rightarrow 2}}$}} &
{\makecell{$\mathbf{L^{Topo}_{2\rightarrow T\rightarrow 1}}$}} \\
\midrule
\textbf{CDL-A,C}              & {-16.7} & {-30.0} & {5.85} & {28.28} & {-11.27} & {-11.58} & {-0.31} & {0.56} & \textbf{-14.88} & \textbf{-12.67} \\
\textbf{UMa}              & {-4.48} & {-7.31} & {5.19} & {32.15} & {-2.68} & {-2.29} & {0.04} & {1.61} & \textbf{-3.54} & \textbf{-2.96} \\

\bottomrule
\end{tabular}
\end{threeparttable}
\end{adjustbox}
\end{table*}

Table~\ref{tab:results_csi} shows that direct cross-decoding fails substantially across all CSI datasets, indicating that independently trained representations are not immediately compatible. Fitting a lightweight affine connector improves performance, but a gap remains relative to within-model decoding. In contrast, topology-preserving training consistently improves stitched performance, reducing reconstruction error. Notably, the gains persist in the quantized setting, suggesting that the regularizer encourages a representation geometry that is robust to discretization artifacts.

\subsection{Sketch-based image retrieval}
\label{sec:results_fscoco}

To test the robustness of our framework beyond autoencoder-style reconstruction, we evaluate the method on a cross-modal retrieval setting of sketch-based image retrieval (SBIR). The objective of SBIR is to embed sketches and images into a common latent space in which images semantically similar to a hand-drawn sketch query can be retrieved. We build our evaluation framework based on the recent SBIR implementation in \cite{SBIR_method} and leverage the \textbf{FS-COCO} \cite{FSCOCOdata} dataset. We first train 2 models end-to-end, independently. The first framework uses the ViT-B model for learning the sketch and image embeddings, and the second framework uses the ConvNeXt \cite{liu2022convnet}. Based on the implementation in \cite{SBIR_method}, models pre-trained on ImageNet-1K are utilized and then fine-tuned independently on the FS-COCO dataset with their proposed ICon loss function as the base loss metric.  Post-training, one model produces sketch embeddings, and the other produces the image embedding. Table~\ref{tab:results_fscoco} reports mAP and Top-$k$ retrieval accuracy.
\begin{table*}[t]
\centering
\caption{FS-COCO sketch-based image retrieval. Higher is better. Affine denotes fitting an affine connector on a calibration set.}
\label{tab:results_fscoco}
\begin{adjustbox}{max width=\linewidth}
\begin{threeparttable}
\begin{tabular}{l S[table-format=1.3] S[table-format=1.3] S[table-format=1.3] S[table-format=1.3] S[table-format=1.4] S[table-format=1.4] S[table-format=1.4] S[table-format=1.4]}
\toprule
\textbf{Metric} &
{$\mathbf{L_{11}}$} & $\mathbf{L_{22}}$ & $\mathbf{L_{12}}$ & $\mathbf{L_{21}}$ &
{\makecell{$\mathbf{L^{DAA}_{1\rightarrow T\rightarrow 2}}$ }} &
{\makecell{$\mathbf{L^{DAA}_{2\rightarrow T\rightarrow 1}}$ }} &
{\makecell{$\mathbf{L^{Topo}_{1\rightarrow T\rightarrow 2}}$}} &
{\makecell{$\mathbf{L^{Topo}_{2\rightarrow T\rightarrow 1}}$}} \\
\midrule
\textbf{mAP}    & {0.636} & {0.524} & {0.002} & 0.002 & 0.348 & 0.317 & 0.378 & 0.388 \\
\textbf{Top-1}  & {53.7} & {41.9} & {0.0} & {0.0} & {23.0} & {20.71} & {26.27} & {27.12} \\
\textbf{Top-5}  & {75.6} & {64.5} & {0.1} & {0.1} & {47.3} & {43.12} & {50.54} & {51.52} \\
\textbf{Top-10} & {82.4} & {73.3} & {0.3} & {0.2} & {59.0} & {54.59} & {61.76} & {62.86} \\
\bottomrule
\end{tabular}
\end{threeparttable}
\end{adjustbox}
\end{table*}

Without alignment, cross-modal retrieval is essentially degenerate ($L_{12},L_{21}\approx 0$ across metrics),
indicating that independently trained sketch and image embedding spaces are not directly comparable.
An affine connector yields strong gains (mAP $0.3486/0.3171$; Top-10 $0.5905/0.5459$), and the proposed topology-preserving training improves the accuracies further. Especially, in the $L^{Topo}_{2\rightarrow T\rightarrow 1}$ case, with the ViT-based sketch encoder and Conv-based image encoder, the top-$1$ to top-$10$ retrieval accuracies are improved by $6$-$8\%$. 
These results highlight the capability of topology-preserving regularization in promoting interoperability for applications beyond reconstruction.

\subsection{Geometric diagnostics: tree-metric discrepancy}
\label{sec:results_geometry}

We next examine whether topology-preserving training moves representations toward the conditions used in our
analysis and promotes agreement of the \emph{induced tree-path metric} across independently trained latent spaces. To probe this mechanism empirically, we introduce a diagnostic that measures cross-model \emph{tree-metric discrepancy}. 
The tree-metric discrepancy measure is \emph{directly tied} to the mechanism enforced by the topological loss: if both encoders preserve the MST edge lengths of the same reference tree, then their induced tree-path distances must agree for all pairs, up to accumulated edge-wise discrepancies (Lemma~\ref{lem:tree_metric}).

\paragraph{Tree-metric discrepancy.}
We first sample a fixed held-out set $\{x_i\}_{i=1}^n$ and construct a reference tree
$T_X=(V,P_X)$ by computing the minimum spanning tree (MST) in input space, where $V=\{1,\dots,n\}$
and $P_X$ is the set of MST edges.
For each encoder $E_k$ ($k\in\{1,2\}$), we compute latents $z_i^{(k)}=E_k(x_i)$ and define the induced tree-path distance
\[
d^{Z_k}_{T_X}(i,j) := \sum_{(u,v)\in\mathrm{path}_{T_X}(i,j)} \|z_u^{(k)}-z_v^{(k)}\|_2.
\]

where, $\mathrm{path}_{T_X}(i,j)$ denotes the unique simple path between $i$ and $j$ in the tree.
We quantify the cross-model tree-metric discrepancy by
\[
\delta_{{tree}}(i,j) := \left| d^{Z_1}_{T_X}(i,j) - d^{Z_2}_{T_X}(i,j)\right|,
\]
and report summary statistics (mean / 95th / 99th percentile) over 20k sampled pairs $(i,j)$ in Table~\ref{tab:tau_results}. 

Across datasets, topology-preserving training leads to a pronounced reduction in the cross-space tree-metric discrepancy $\bm{\delta_{tree}}$. This pattern is consistent with the role of the topological loss: it encourages both models to induce similar path-length structure on the same MST skeleton, which is the quantity directly controlled by Lemma~\ref{lem:tree_metric}. A reduction in $\bm{\delta_{tree}}$ indicates that independently trained latent spaces encode a \emph{shared} topological backbone more consistently, making them easier to interconnect with a lightweight affine connector. In our experiments, topology-preserving training substantially reduces $\delta_{\mathrm{tree}}$ across datasets, and this reduction correlates with improved stitched performance.

\begin{table}[t]
\centering
\caption{Tree-metric discrepancy diagnostics.}
\label{tab:tau_results}
\begin{adjustbox}{max width=\linewidth}
\begin{threeparttable}
\begin{tabular}{l
S[table-format=2.2] S[table-format=2.2] S[table-format=2.2]
S[table-format=2.2] S[table-format=2.2] S[table-format=2.2]
S[table-format=4.1] S[table-format=4.1] S[table-format=4.1]
S[table-format=4.1] S[table-format=4.1] S[table-format=4.1]
}
\toprule
\multirow{2}{*}{\textbf{Dataset}} &
\multicolumn{3}{c}{\textbf{$\bm{\delta_{tree}}$ (no topo)}} &
\multicolumn{3}{c}{\textbf{$\bm{\delta_{tree}}$ (topo)}} \\
\cmidrule(lr){2-4}\cmidrule(lr){5-7}\cmidrule(lr){8-10}\cmidrule(lr){11-13}
& \textbf{Mean} & \textbf{95} & \textbf{99} & \textbf{Mean} & \textbf{95} & \textbf{99} \\
\midrule

\textbf{MNIST}    & {290.8} & {521.9} & {592.1} & {192.8} & {345.6} & {392.4} \\
\textbf{CIFAR}    & {54.4} & {90.0} & {107.0} & {13.0} & {21.3} & {25.0} \\
\textbf{T-ImageNet} & {542.8} & {927.2} & {1109.8} & {84.9} & {151.3} & {185.2} \\
\textbf{CelebA} & {133.2} & {254.3} & {298.7} & {99} & {188.9} & {218.9} \\

\bottomrule
\end{tabular}
\end{threeparttable}
\end{adjustbox}
\end{table}

\subsection{Ablation studies}
\label{sec:ablations}

We conduct ablations to assess (i) sensitivity to the topology regularization weight $\lambda$
and (ii) sample efficiency of the affine connector fitting stage (calibration set size $|D_T|$).
Unless stated otherwise, all other training and evaluation settings are identical to the main experiments.

\paragraph{Effect of topology regularization weight.}
Figures~\ref{fig:ablation_mnist} and~\ref{fig:ablation_csi} vary the
topology regularization weight $\lambda$ in \eqref{eq:d2l}.
We observe that performance improves substantially over the $\lambda \approx 0$ regime (which reduces to
standard training without topology preservation), and remains stable across a broad range of moderate
$\lambda$ values. For excessively large $\lambda$, performance can degrade, consistent with
over-regularization: enforcing topology too strongly can compete with reconstruction objectives and
reduce representational flexibility.
Overall, these ablations indicate that topology-preserving training is not brittle, and that selecting
$\lambda$ within a moderate range yields consistent gains.

\paragraph{Effect of calibration size for affine stitching.}
Figures~\ref{fig:ablation_mnist} and~\ref{fig:ablation_csi} also vary the number of calibration
examples used to fit the affine connector $T$ (i.e., $|D_T|$ in \eqref{eq:direct-alignment}).
Across both MNIST and CSI, topology-preserving training improves \emph{sample efficiency}: for a fixed
calibration budget, stitched performance is consistently better with topology than without.
Moreover, the marginal benefit of increasing $|D_T|$ diminishes beyond a moderate regime, suggesting that the principal role of the connector is to correct residual global misalignment once the training-time regularizer has aligned the two representations on a common geometric scaffold.

\begin{figure}[t]
\centering
\begin{minipage}[t]{0.48\linewidth}
  \centering
  \includegraphics[width=\linewidth]{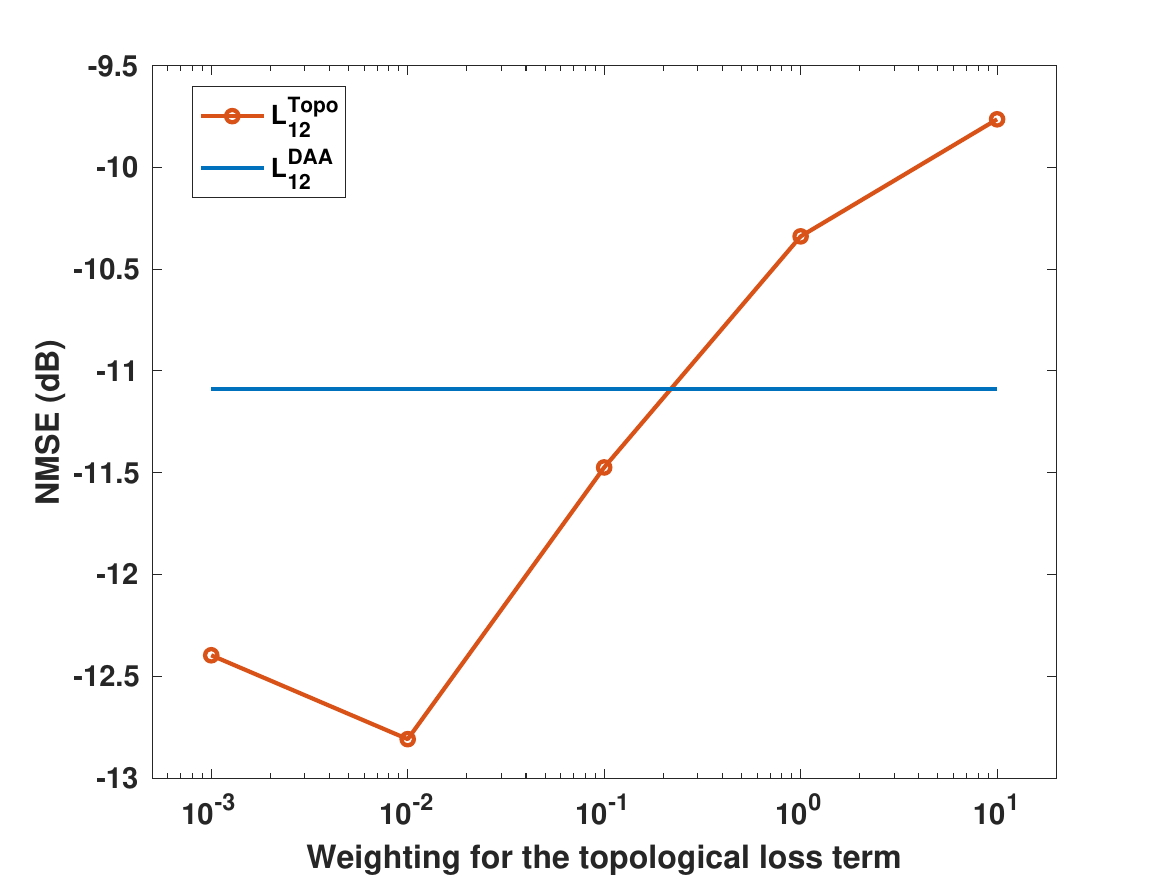}
  \par\smallskip\textbf{(a)} Effect of $\lambda$.
\end{minipage}
\hfill
\begin{minipage}[t]{0.48\linewidth}
  \centering
  \includegraphics[width=\linewidth]{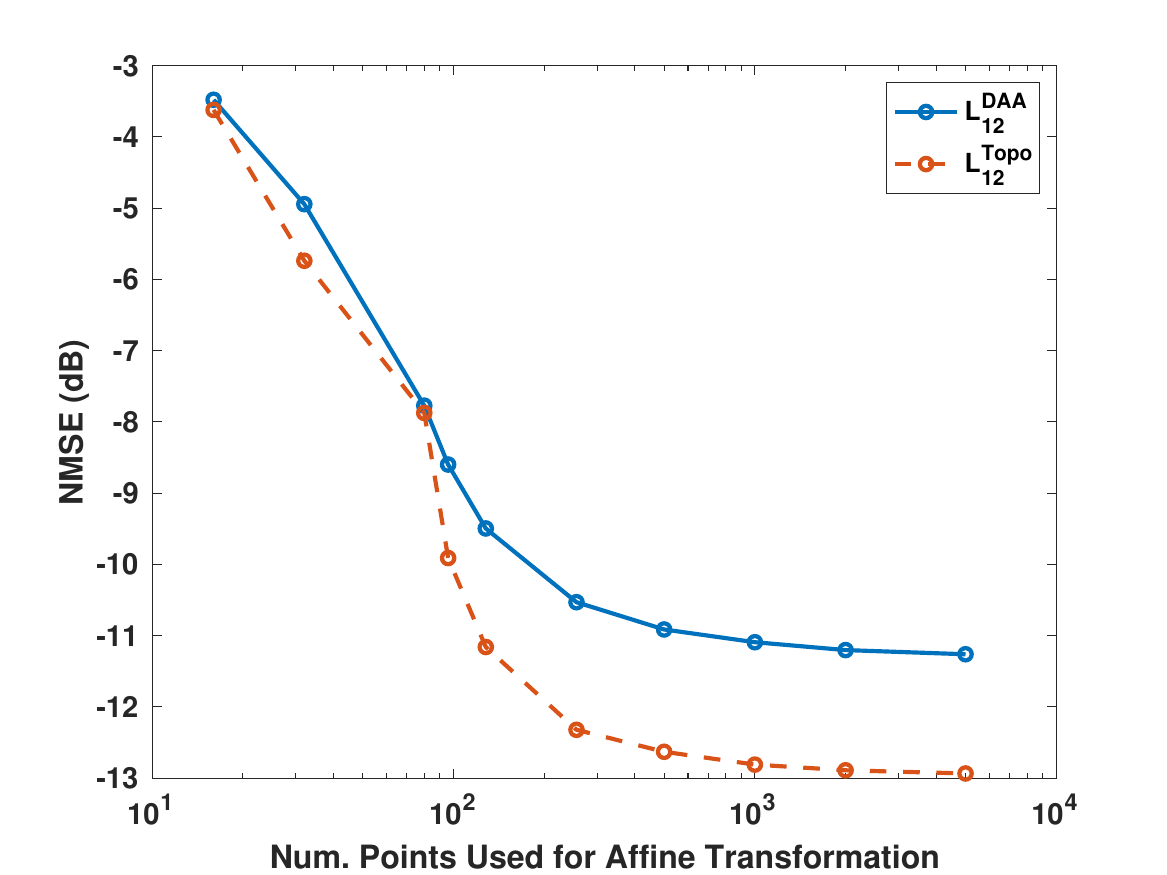}
  \par\smallskip\textbf{(b)} Effect of $|D_T|$.
\end{minipage}
\caption{MNIST ablations: sensitivity to topology weight $\lambda$ and calibration size $|D_T|$.}
\label{fig:ablation_mnist}
\end{figure}

\begin{figure}[t]
\centering
\begin{minipage}[t]{0.48\linewidth}
  \centering
  \includegraphics[width=\linewidth]{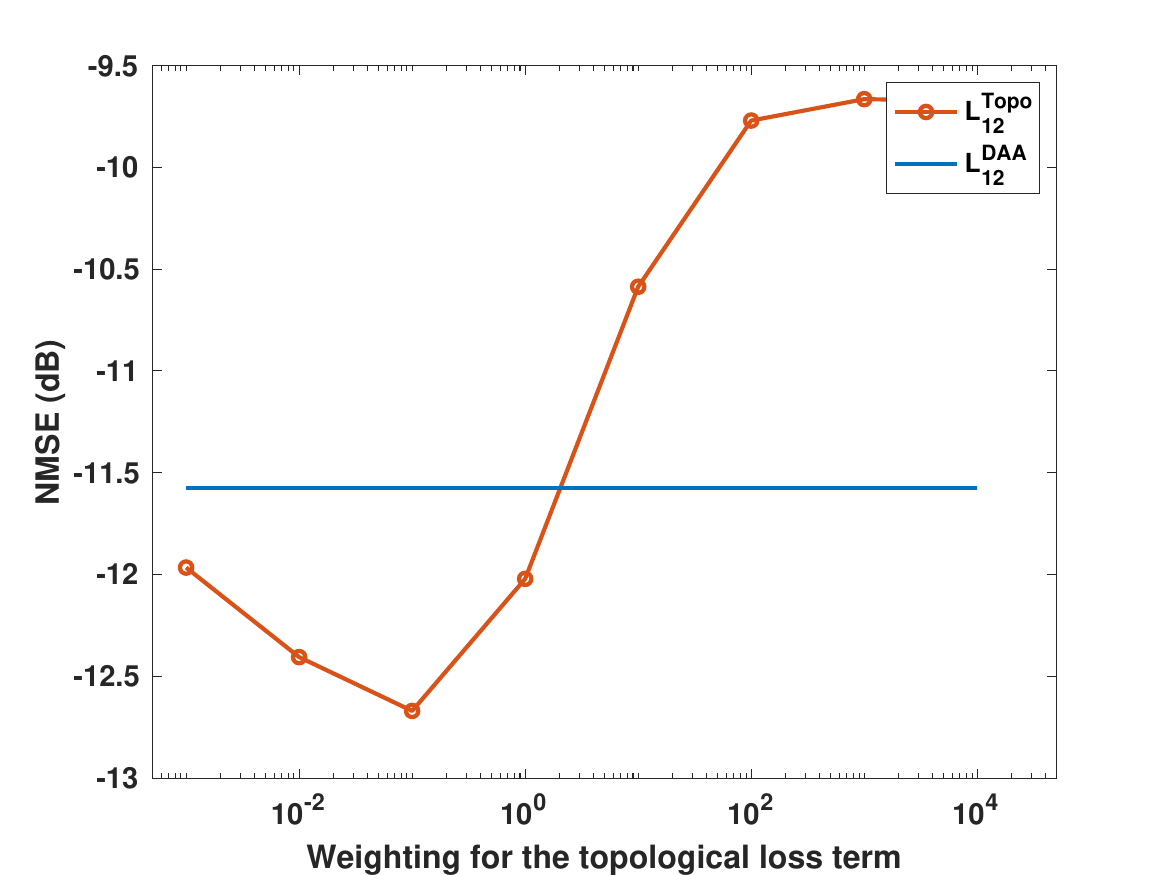}
  \par\smallskip\textbf{(a)} Effect of $\lambda$.
\end{minipage}
\hfill
\begin{minipage}[t]{0.48\linewidth}
  \centering
  \includegraphics[width=\linewidth]{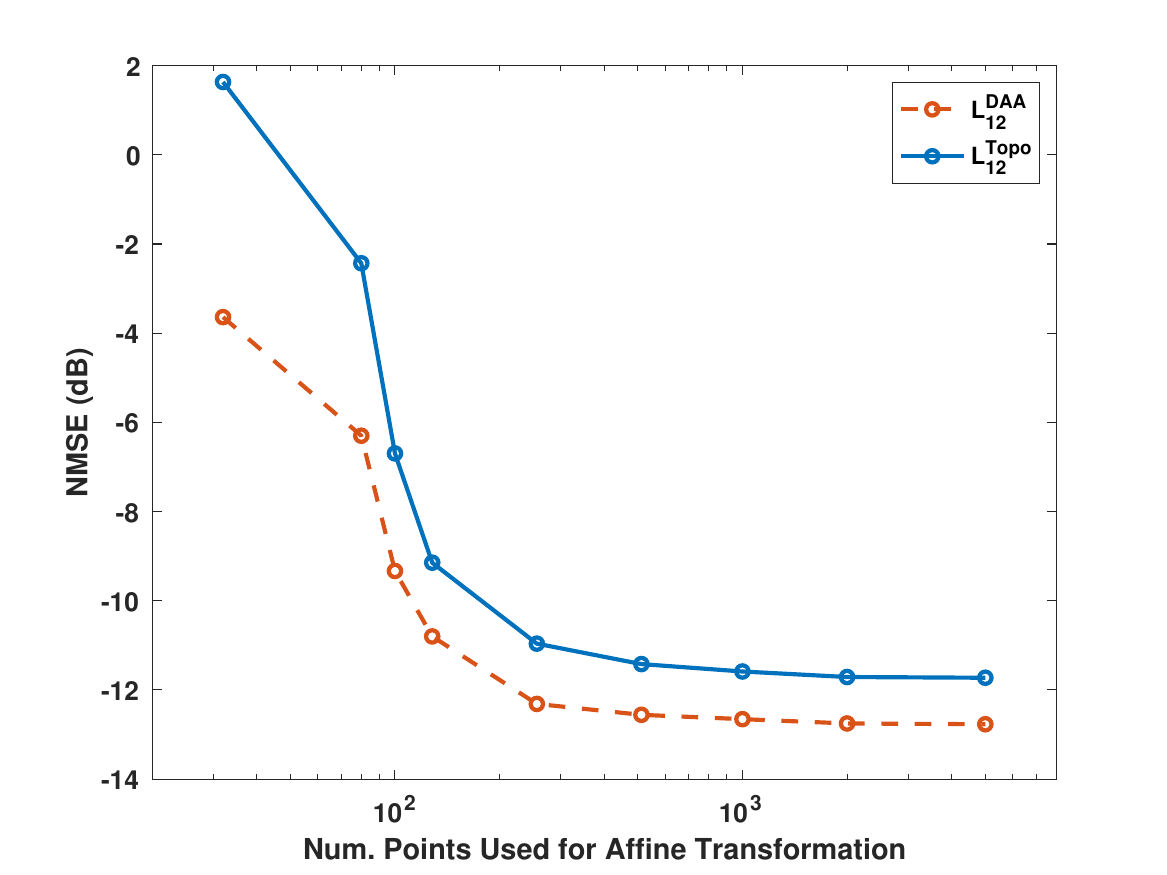}
  \par\smallskip\textbf{(b)} Effect of $|D_T|$.
\end{minipage}
\caption{CSI ablations: sensitivity to topology weight $\lambda$ and calibration size $|D_T|$.}
\label{fig:ablation_csi}
\end{figure}

\section{Conclusion}

We investigate the problem of zero-shot model stitching in autoencoders, demonstrating that topological regularization promotes a simplified, affine mapping between independent latent spaces. Our proposed framework introduces a training scheme aimed at preserving the topological structure of the input data within the latent space of an otherwise unconstrained autoencoder. We provide theoretical justification for how this topological loss controls agreement between induced tree-path metrics. Under the additional condition of bounded tree stretch, this agreement implies approximate Euclidean pairwise-distance agreement between representations. Through numerical experiments across heterogeneous domains, including both image and wireless datasets, we show that our method drastically simplifies latent-space alignment. Future work will explore alternative geometric losses to further refine the alignment of topological features.

\bibliographystyle{splncs04}
\bibliography{citations}
\appendix
\input{appendix.tex}
\end{document}

%% file: appendix.tex
\section{Proofs}
\label{app:proofs}
\subsection{Proof for Lemma 1}

\begin{proof}
Let $\bbH=\bbI-\tfrac{1}{n}\mathbf{1}\mathbf{1}^\top$ be the centering matrix. A classical identity in
multidimensional scaling (MDS) states that the (centered) Gram matrix can be recovered from squared
distances via double centering:
\begin{equation}
\label{eq:mds_double_center}
\mathbf{Z}\mathbf{Z}^\top \;=\; -\tfrac{1}{2}\,\bbH\,\Delta(\mathbf{Z})\,\bbH,
\end{equation}
and similarly $\mathbf{Y}\mathbf{Y}^\top = -\tfrac{1}{2}\,\bbH\,\Delta(\mathbf{Y})\,\bbH$.

Since $\Delta(\mathbf{Z})=\Delta(\mathbf{Y})$, applying \eqref{eq:mds_double_center} to both point clouds yields
\[
\mathbf{Z}\mathbf{Z}^\top \;=\; -\tfrac{1}{2}\,\bbH\,\Delta(\mathbf{Z})\,\bbH
\;=\; -\tfrac{1}{2}\,\bbH\,\Delta(\mathbf{Y})\,\bbH
\;=\; \mathbf{Y}\mathbf{Y}^\top.
\]
Denote this common Gram matrix by $\bbG\in\mathbb{R}^{n\times n}$, so $\bbG=\mathbf{Z}\mathbf{Z}^\top=\mathbf{Y}\mathbf{Y}^\top$.


Assume $\mathbf{Z}$ has full column rank $d$ (hence $\mathrm{rank}(\bbG)=d$). Take a thin SVD of $\mathbf{Z}$:
\[
\mathbf{Z}=\bbU\bbSigma \bbV^\top,
\]
where $\bbU\in\mathbb{R}^{n\times d}$ and $\bbV\in\mathbb{R}^{d\times d}$ have orthonormal columns and
$\bbSigma\in\mathbb{R}^{d\times d}$ is diagonal with positive entries. Then
\[
\bbG=\mathbf{Z}\mathbf{Z}^\top = \bbU\bbSigma \bbV^\top \bbV\bbSigma \bbU^\top = \bbU\bbSigma^2 \bbU^\top.
\]
Since $\mathbf{Y}\mathbf{Y}^\top=G$ as well and $\mathrm{rank}(\mathbf{Y})=\mathrm{rank}(\bbG)=d$, a thin SVD of
$\mathbf{Y}$ can be written as
\[
\mathbf{Y}=\bbU\bbSigma \bbW^\top
\]
for some orthogonal $\bbW\in\mathbb{R}^{d\times d}$ (the left singular vectors $\bbU$ and singular values $\Sigma$
are determined by $\bbG=\bbU\bbSigma^2\bbU^\top$ up to an orthogonal change of basis in $\mathbb{R}^d$).
Therefore,
\[
\begin{aligned}
\mathbf{Y}
  &=\bbU\bbSigma \bbW^\top\\
  &=(\bbU\bbSigma \bbV^\top)(\bbV \bbW^\top)
   =\mathbf{Z}\bm{Q},\\
\bm{Q}&:=\bbV \bbW^\top.
\end{aligned}
\]
Since $V$ and $W$ are orthogonal, $\bm{Q}$ is orthogonal as well: $\bm{Q}^\top \bm{Q} = \bbI$.
\end{proof}

\subsection{Proof for Lemma 2}
\begin{proof}
Fix $(i,j)$ and let $\mathrm{path}_{T_X}(i,j)=\{e_1,\dots,e_m\}$ be the (ordered) edge set along the unique tree path from $i$ to $j$, where $m=|\mathrm{path}_{T_X}(i,j)|$.
Write each edge as $e_\ell=(u_{\ell},v_{\ell})\in P_X$, and set
\[
\delta_\ell :=
\|z_{u_\ell}-z_{v_\ell}\|_2
-\|y_{u_\ell}-y_{v_\ell}\|_2.
\]
Then
\begin{align*}
\big| d_{T_X}^{Z}(i,j) - d_{T_X}^{Y}(i,j) \big|
&= \left|\sum_{\ell=1}^m \delta_\ell\right| \\
&\le \sum_{\ell=1}^m |\delta_\ell|.
\end{align*}
Define $a\in\mathbb{R}^{|P_X|}$ indexed by edges $e=(u,v)\in P_X$ as
\[
\begin{aligned}
a_e &:= \|z_u-z_v\|_2 - \|y_u-y_v\|_2,\\
a   &:= A_Z[P_X]-A_Y[P_X].
\end{aligned}
\]
The sum above is over the $m$ coordinates $\{|a_{e_\ell}|\}_{\ell=1}^m$, hence by Cauchy--Schwarz,
\[
\sum_{\ell=1}^m |a_{e_\ell}| \le \sqrt{m}\,\|a\|_2
= \sqrt{m}\,\|A_Z[P_X]-A_Y[P_X]\|_2.
\]
Since $m\le \mathrm{diam}(T_X)$ for all $(i,j)$, taking the maximum over $(i,j)$ yields the claim.
\end{proof}

\subsection{Proof for Theorem 1}

\begin{proof}
For any pair $(i,j)$, decompose
\begin{equation}
\begin{aligned}
&\|z^{(1)}_i-z^{(1)}_j\|_2
 -\|z^{(2)}_i-z^{(2)}_j\|_2\\
&=\big(\|z^{(1)}_i-z^{(1)}_j\|_2
       -d^{Z_1}_{T_X}(i,j)\big)\\
&\quad+\big(d^{Z_1}_{T_X}(i,j)
            -d^{Z_2}_{T_X}(i,j)\big)\\
&\quad+\big(d^{Z_2}_{T_X}(i,j)
            -\|z^{(2)}_i-z^{(2)}_j\|_2\big).
\end{aligned}
\end{equation}
The first and last terms are bounded by stretch since $\|z_i-z_j\|_2\le d^{Z}_{T_X}(i,j)\le
\tau\|z_i-z_j\|_2\le \tau R$, giving
$0\le d^{Z}_{T_X}(i,j)-\|z_i-z_j\|_2\le(\tau-1)R$ for $Z\in\{Z_1,Z_2\}$.
The middle term is controlled by Lemma 2 together with the triangle inequality and
the definitions of $\varepsilon_1,\varepsilon_2$ (both $Z_1$ and $Z_2$ match the same target edge
lengths on $P_X$ up to $\varepsilon_k$).
This yields the uniform bound on absolute distance differences.
The squared-distance matrix bound follows from $|a^2-b^2|=|a-b||a+b|\le 2R|a-b|$ and summing over all
unordered pairs.
\end{proof}


\section{Experimental Details}
\label{app:exp_details}

\subsection{CIFAR-10}
\textbf{Data} RGB images of size $3\times 32\times 32$; pixel intensities rescaled to $[0,1]$. 

\noindent\textbf{Models}

\emph{Linear AE:} flatten input to $3072$. Encoder MLP
$3072\!\rightarrow\!256\!\rightarrow\!128\!\rightarrow\!128\!\rightarrow\!d$
with ReLU after each layer; decoder
$d\!\rightarrow\!128\!\rightarrow\!128\!\rightarrow\!256\!\rightarrow\!3072$
with ReLU in hidden layers and a sigmoid output (reshaped to $3\times 32\times 32$).

\emph{Conv AE:} 4 strided conv blocks with channels
$3\!\rightarrow\!16\!\rightarrow\!32\!\rightarrow\!64\!\rightarrow\!d/4$
(ReLU), then flatten to $z\in\mathbb{R}^d$; decoder uses 4 transposed-conv blocks
$d/4\!\rightarrow\!64\!\rightarrow\!32\!\rightarrow\!16\!\rightarrow\!3$
with ReLU and a sigmoid output.

\noindent\textbf{Training} Adam, learning rate $10^{-3}$, batch size $128$, $200$ epochs. The best validation checkpoint is retained. $5$ independent runs. Latent size 64.

\noindent\textbf{Stitching} Affine transformation estimated using $1000$ samples.

\subsection{MNIST}

\textbf{Data:} Grayscale images of size $1\times 28\times 28$; pixel intensities rescaled to $[0,1]$.

\noindent\textbf{Models:}

\emph{Linear AE:} flatten input to $784$. Encoder MLP
$784\!\rightarrow\!256\!\rightarrow\!128\!\rightarrow\!d$
with ReLU after each layer; decoder
$d\!\rightarrow\!128\!\rightarrow\!256\!\rightarrow\!784$
with ReLU in hidden layers and a sigmoid output (reshaped to $1\times 28\times 28$).

\emph{Conv AE:} stride-2 conv stack with channels
$1\!\rightarrow\!16\!\rightarrow\!32\!\rightarrow\!64\!\rightarrow\!d/4$
(ReLU), then flatten to $z\in\mathbb{R}^d$; decoder uses transposed-conv stack
$d/4\!\rightarrow\!64\!\rightarrow\!32\!\rightarrow\!16\!\rightarrow\!1$
with ReLU and a sigmoid output.

\noindent\textbf{Training:} Adam, learning rate $10^{-3}$, batch size $128$, $150$ epochs, Kaiming initialization. The best validation checkpoint is retained. $5$ independent runs. Latent size 64.

\noindent\textbf{Stitching:} Affine transformation estimated using $1000$ samples.

\subsection{Tiny-ImageNet}

\textbf{Data:} TinyImageNet-200 RGB images of size $3\times 64\times 64$; rescaled to $[0,1]$. 

\noindent\textbf{Models:}

\emph{Conv AE:} encoder with 4 strided conv blocks
$3\!\rightarrow\!32\!\rightarrow\!64\!\rightarrow\!128\!\rightarrow\!256$
using BN+GELU, followed by an FC layer to latent $z\in\mathbb{R}^{64}$; decoder uses FC + 4 transposed-conv blocks
$256\!\rightarrow\!128\!\rightarrow\!64\!\rightarrow\!32\!\rightarrow\!3$
with BN+GELU, output is linear clamped to $[0,1]$.

\emph{ViT AE:} patch embedding with \emph{patch size} $8$, $4$ transformer encoder blocks, then features pooled together, then FC to $z\in\mathbb{R}^{64}$; transformer decoder reconstructs image patches. Embedding dim $256$, $4$ encoder transformer blocks, $2$ decoder transformer blocks, $8$ heads.

\noindent\textbf{Training:} Adam, learning rate $2\times 10^{-4}$, batch size $512$, $100$ epochs, Kaiming initialization. The best validation checkpoint is retained. $5$ independent runs. Latent size 64.

\noindent\textbf{Stitching:} Affine transformation estimated using $1000$ samples.

\subsection{CelebA}

\textbf{Data:} CelebA images resized to $3\times 128\times 128$ and converted pixel range $[0,1]$); No augmentation or normalization. 

\noindent\textbf{Models:}

\emph{Conv AE:} encoder with 5 strided conv blocks (GELU) with channels
$3\!\rightarrow\!32\!\rightarrow\!64\!\rightarrow\!128\!\rightarrow\!256\!\rightarrow\!256$
and spatial resolution
$128\!\rightarrow\!64\!\rightarrow\!32\!\rightarrow\!16\!\rightarrow\!8\!\rightarrow\!4$,
then FC to latent $z\in\mathbb{R}^d$; decoder uses FC to $256\times 4\times 4$ followed by 5 transposed-conv blocks (GELU) with a final sigmoid output.

\emph{ViT AE:} patch-based transformer encoder/decoder with patch size $8$; Embedding dim $192$, $4$ encoder transformer blocks, $4$ decoder transformer blocks, $6$ heads; latent $z\in\mathbb{R}^d$.

\noindent\textbf{Training:} Adam, learning rate $10^{-4}$, batch size $128$, $100$ epochs. The best validation checkpoint is retained. $5$ independent runs. Latent size $64$.

\noindent\textbf{Stitching:} Affine transformation estimated using $1000$ samples.

\subsection{CDL-A,C and UMa}

\textbf{Data:} CSI tensors are normalized using training-set mean/std, and represented as two channels (real/imag).

\noindent\textbf{Models:} End-to-end is performed as \texttt{Encoder} $\rightarrow$ \texttt{Quantization} $\rightarrow$ \texttt{Decoder}, with a straight-through quantizer using $2$-bit uniform scalar quantization. 

\noindent\textbf{Training:} Adam, initial learning rate $10^{-3}$, batch size $128$, $150$ epochs, with exponential LR decay every $5$ epochs (factor $0.9$). The best checkpoint is selected by the lowest held-out validation loss. $5$ independent runs.

\noindent\textbf{Stitching:} Affine transformation estimated using $1000$ samples.